\def\year{2021}\relax
\documentclass[letterpaper]{article} 
\usepackage{aaai21}  
\usepackage{times}  
\usepackage{helvet} 
\usepackage{courier}  
\usepackage[hyphens]{url}  
\usepackage{graphicx} 
\usepackage{graphicx}  
\usepackage{natbib}  
\usepackage{caption} 
\usepackage[tbtags]{amsmath}
\usepackage[super]{nth}
\usepackage{mathtools,layout,graphicx,framed,xpatch,multirow,tabularx,amssymb,amsthm,enumitem,bbm,algorithm,caption,subcaption,adjustbox,booktabs,setspace,svg,algcompatible}
\usepackage[switch]{lineno}  
\newtheorem{theorem}{Theorem}
\newtheorem{lemma}{Lemma}
\newtheorem{definition}{Definition}
\DeclareMathOperator*{\argmax}{arg\!\max}

\newcommand{\ignore}[1]{}
\newcommand{\frank}[1]{\textbf{\textcolor{red}{Frank: #1}}}
\nocopyright
\title{Pareto Optimization for Subset Selection with Dynamic Partition Matroid Constraints}
\author{
    Anh Viet Do\textsuperscript{\rm 1}, Frank Neumann\textsuperscript{\rm 1}
    \\
}
\affiliations{
    \textsuperscript{\rm 1}Optimisation and Logistics, School of Computer Science, The University of Adelaide, Adelaide, Australia\\
    vietanh.do@adelaide.edu.au, frank.neumann@adelaide.edu.au

}

\begin{document}
\maketitle

\begin{abstract}
In this study, we consider the subset selection problems with submodular or monotone discrete objective functions under partition matroid constraints where the thresholds are dynamic. We focus on POMC, a simple Pareto optimization approach that has been shown to be effective on such problems. Our analysis departs from singular constraint problems and extends to problems of multiple constraints. We show that previous results of POMC's performance also hold for multiple constraints. Our experimental investigations on random undirected maxcut problems demonstrate POMC's competitiveness against the classical GREEDY algorithm with restart strategy.
\end{abstract}

\section{Introduction}
Many important real-world problems involve optimizing a submodular function. Such problems include maximum coverage, maximum cut \cite{MaxcutSDP}, maximum influence \cite{MaxSpreadInfluence}, sensor placement problem \cite{NearSensorPlacement,SubInfoGather}, as well as many problems in the machine learning domain \cite{FeatureSelection,DataSubsetSelection,DocSum,DocSumBudget,EfficientMinDecompSubmodular}. Much work has been done in the area of submodular optimization under static constraints. 
 A particularly well-studied class of algorithms in this line of research is greedy algorithms, which have been shown to be efficient in exploiting submodularity \cite{LocationBankOptFloat,MaxSubmodularMatroid,GreedyMaxBoundCurvPartMatroid,GreedNonmod}.
Important recent results on the use of evolutionary algorithms for submodular optimization are summarized in~\cite{DBLP:books/sp/ZhouYQ19}.

Real-world problems are seldom solved once, but rather many times over some period of time, during which they change. Such changes demand adapting the solutions that would otherwise become poor or infeasible. The dynamic nature of these problems presents many interesting optimization challenges, which have long been embraced by many researchers. A lot of research in the evolutionary computation literature has addressed these types of problems from an applied perspective.
Theoretical investigations have been carried out for evolutionary algorithms on some example functions and classical combinatorial optimization problems such as shortest paths, but in general the theoretical understand on complex dynamic problems is rather limited~\cite{DBLP:journals/corr/abs-1806-08547}.
In this paper, we follow the approach of carrying out theoretical runtime analysis of evolutionary algorithms with respect to their runtime and approximation behavior. This well established area of research has significantly increased the theoretical understanding of evolutionary computation methods~\cite{DBLP:books/daglib/0025643,BookDoeNeu}.

Many recent studies on submodular and near-submodular optimization have investigated Pareto optimization approaches based on evolutionary computation. \citet{POMC} derive an approximation guarantee for the POMC algorithm for maximizing monotone function under a monotone constraint. They show that POMC achieves an $(\alpha_f/2)(1-1/e^{\alpha_f})$-approximation within at most cubic expected run time. The recent study of \citet{MaxMonoApproxSubmodMulti} extends the results to a variant of the GSEMO algorithm (which inspired POMC) to the problem of maximizing general submodular functions, but under a cardinality constraint. It results reveal that non-monotonicity in objective functions worsens approximation guarantees.

In our work, we extend existing results for POMC \cite{ParetoSubsetDynCon} to partition matroid constraints with dynamic thresholds. We show that the proven adaptation efficiency facilitated by maintaining dominating populations can be extended for multiple constraints with appropriately defined dominance relations. In particular, we prove that POMC can achieve new approximation guarantees quickly whether the constraints thresholds are tightened or relaxed. Additionally, we study POMC experimentally on the dynamic max cut problem and compare its results against the results of greedy algorithms for underlying static problems. Our study evaluates the efficiency in change adaptation, thus assuming immaculate change detection. Our results show that POMC is competitive to GREEDY during unfavorable changes, and outperforming GREEDY otherwise.

In the next section, we formulate the problem and introduce the Pareto optimization approach that is subject to our investigations. Then we analyze the algorithm in terms of runtime and approximation behaviour when dealing with dynamic changes. Finally, we present the results of our experimental investigations and finish with some conclusions.
\section{Problem Formulation and Algorithm}

In this study, we consider optimization problems where the objective functions are either submodular or monotone. We use the following definition of submodularity \cite{Submod1978}.
\begin{definition}\label{submodular_def}
Given a finite set $V$, a function $f:2^V\to\mathbb{R}^+$ is submodular if it satisfies for all $X\subseteq Y\subseteq V$ and $v\in V\setminus Y$,
\[f(Y\cup\{v\})-f(Y)\leq f(X\cup\{v\})-f(X).\]
\end{definition}

As in many relevant works, we are interested in the submodularity ratio which quantifies how close a function is to being modular. In particular, we use a simplified version of the definition in \cite{spectral}.

\begin{definition}\label{submodratio_def}
For a monotone function $f:2^V\to\mathbb{R}^+$, its submodularity ratio with respect to two parameters $i$, $j\geq1$ is
\[\gamma_{i,j}=\min_{|X|<i,|L|\leq j,X\cap L=\emptyset}\frac{\sum_{v\in L}[f(X\cup\{v\})-f(X)]}{f(X\cup L)-f(X)},\]
for $i>0$ and $\gamma_{0,j}=\gamma_{1,j}$.
\end{definition}

It can be seen that $\gamma_{i,j}$ is non-negative, non-increasing with increasing $i$ and $j$, and $f$ is submodular iff $\gamma_{i,j}\geq1$ for all $(i,j)$. This ratio also indicates the intensity of the function's diminishing return effect. Additionally, non-monotonicity is also known to affect worst-case performance of algorithms \cite{GreedyMaxBoundCurvPartMatroid,MaxMonoApproxSubmodMulti}. As such, we also use the objective function's monotonicity approximation term defined similarly to \cite{NearSensorPlacement}, but only for subsets of a certain size.

\begin{definition}\label{mono_approx}
For a function $f:2^V\to\mathbb{R}^+$, its monotonicity approximation term with respect to a parameter $j$ is
\[\epsilon_j=\max_{X,v:|X|<j}\{f(X\setminus\{v\})-f(X)\},\]
for $j>0$ and $\epsilon_0=0$.
\end{definition}

It is the case that $\epsilon_j$ is non-negative, non-decreasing with increasing $j$, and $f$ is monotone iff $\epsilon_{n+1}=0$. We find that adding the size parameter can provide extra insight into the analysis results.

Consider the static optimization problem with partition matroid constraints.
\begin{definition}\label{def_static}
Given a set function $f:2^V\to\mathbb{R}^+$, a partitioning $B=\{B_i\}_{i=1}^k$ of $V$, and a set of integer thresholds $D=\{d_i\}_{i=1}^k$, the problem is
\begin{align*}
\underset{X\subseteq V}{\text{maximize }}f(X),
\text{ s.t. }|X\cap B_i|\leq d_i,\quad\forall i=1,\dots,k.
\end{align*}
\end{definition}
We define notations $d=\sum_{i=1}^kd_i$, $\bar{d}=\min_i\{d_i\}$, and $OPT\subseteq V$ the feasible optimal solution. A solution $X\subseteq V$ is feasible iff it satisfies all constraints. It can be shown that $\bar{d}\leq d/k$, $|OPT|\leq d$, and any solution $X$ where $|X|\leq\bar{d}$ is feasible. Each instance is then uniquely defined by the triplet $(f,B,D)$. Without loss of generality, we assume $1\leq d_i\leq |B_i|$ for all $i$.

We study the dynamic version of the problem in Definition \ref{def_static}. This dynamic problem demands adapting the solutions to changing constraints whenever such changes occur.
\begin{definition}\label{def_dynamic}
Given the problem in Definition \ref{def_static}, a dynamic problem instance is defined by a sequence of changes where the current $D$ in each change is replaced by $D^*=\{d^*_i\}_{i=1}^k$ such that $d^*_i\in[1,|B_i|]$ for $i=1,\dots,k$. The problem is to generate a solution $X$ that maximizes $f(X)$ for each newly given $D^*$ such that
\[|X\cap B_i|\leq d^*_i,\quad\forall i=1,\dots,k.\]
\end{definition}

Such problems involve changing constraint thresholds over time. Using the oracle model, we assume time progresses whenever a solution is evaluated. We define notations $d^*=\sum_{i=1}^kd^*_i$, $\bar{d}^*=\min_i\{d^*_i\}$, and the new optimal solution $OPT^*$. Similarly, we assume $1\leq d^*_i\leq |B^*_i|$ for all $i$. Lastly, while restarting from scratch for each new thresholds is a viable tactic for any static problems solver, we focus on the capability of the algorithm to adapt to such changes.

\subsection{POMC algorithm}

The POMC algorithm \cite{POMC} is a Pareto Optimization approach for constrained optimization. It is also known as GSEMO algorithm in the evolutionary computation literature~\cite{SEMO,GSEMO,GSEMO2015}. As with many other evolutionary algorithms, the binary representation of a set solutions is used. For this algorithm, we reformulate the problem as a bi-objective optimization problem given as
\[\text{maximize }(f_1(X),f_2(X)),\]
where
\[f_1(X)=\begin{cases*}
f(X),&if $X$ is feasible\\
-\infty,&otherwise
\end{cases*},\text{ }f_2(x)=-|X|.\]

POMC optimizes two objectives simultaneously, using the dominance relation between solutions, which is common in Pareto optimization approaches. Recall that solution $X_1$ dominates $X_2$ ($X_1\succeq X_2$) iff $f_1(X_1)\geq f_1(X_2)$ and $f_2(X_1)\geq f_2(X_2)$. The dominance relation is strict ($X_1\succ X_2$) iff $X_1\succeq X_2$ and $f_i(X_1)>f_i(X_2)$ for at least one $i \in \{1,2\}$. Intuitively, dominance relation formalizes the notion of ``better'' solution in multi-objective contexts. Solutions that don't dominate any other present a trade-off between objectives to be optimized.

The second objective in POMC is typically formulated to promote solutions that are ``further'' from being infeasible. The intuition is that for those solutions, there is more room for feasible modification, thus having more potential of becoming very good solutions. For the problem of interest, one way of measuring ``distance to infeasibility'' for some solution $X$ is counting the number of elements in $V\setminus X$ that can be added to $X$ before it is infeasible. The value then would be $d-|X|$, which is the same as $f_2(X)$ in practice. Another way is counting the minimum number of elements in $V\setminus X$ that need to be added to $X$ before it is infeasible. The value would then be $\min_i\{d_i-|B_i\cap X|\}$. The former approach is chosen for simplicity and viability under weaker assumptions about the considered problem.

On the other hand, the first objective aims to present the canonical evolutionary pressure based on objective values. Additionally, $f_1$ also discourages all infeasible solutions, which is different from the formulation in \cite{POMC} that allows some degree of infeasibility. This is because for $k>1$, there can be some infeasible solution $Y$ where $|Y|\leq d$. If $f_1(Y)$ is very high, it can dominate many good feasible solutions, and may prevent acceptance of global optimal solutions into the population. Furthermore, restricting to only feasible solutions decreases the maximum population size, which can improve convergence performance. As a consequence, the population size of POMC for our formulation is at most $d+1$. Our formulation of the two objective functions is identical to the ones in \cite{GSEMO2015} when $k=1$.

\begin{algorithm}[t]
\begin{algorithmic}[1]
\STATEx \textbf{Input:} a problem instance: $(f,B,D)$
\STATEx \textbf{Parameter:} the number of iterations $T\geq0$
\STATEx \textbf{Output:} a feasible solution $x\in\{0,1\}^n$
\STATE $x\gets0^n$, $P\gets\{x\}$
\WHILE{$t<T$}
\IF{Change is detected}
\STATE $P\gets P\setminus\{x\in P|\exists y\in P,y\neq x\wedge y\succeq x\}$
\ENDIF
\STATE Randomly sample a solution $y$ from $P$
\STATE $y'\gets y$ after flipping each bit with probability $1/n$
\IF{$\nexists x\in P,x\succ y'$}
\STATE $P\gets(P\setminus\{x\in P|y'\succeq x\})\cup\{y'\}$
\ENDIF
\ENDWHILE
\STATE \textbf{return} $\argmax_{x\in P}f_1(x)$
\end{algorithmic}
\caption{POMC algorithm for dynamic problems}
\label{alg:gsemo}
\end{algorithm}
POMC (see Algorithm~\ref{alg:gsemo}) starts with initial population consisting of the search point $0^n$ which represents the empty set. In each iteration, a new solution is generated by random parent selection and bit flip mutation. Then the elitist survivor selection mechanism removes dominated solutions from the population, effectively maintaining a set of trade-off solutions for the given objectives. The algorithm terminates when the number of iteration reaches some predetermined limit. We choose empty set as the initial solution, similar to \cite{POMC} and different from \cite{GSEMO2015}, to simplify the analysis and stabilize theoretical performance. Note that POMC calls the oracle once per iteration to evaluate a new solution, so its run time is identical to the number of iterations.

We assume that changes are made known to the algorithm as they occur, and that feasibility can be checked efficiently. The reason for this is that infeasibility induced by changes in multiple thresholds has nontrivial impact on the algorithm's behaviour. For single constraint problems, this impact is limited to increases in population size \cite{ParetoSubsetDynCon}, since a solution's degree of feasibility entirely correlates with its second objective value. However, this is no longer the case for multiple constraints, as the second objective aggregates all constraints. While it reduces the population size as the result, it also allows for possibilities where solutions of small size (high in second objective) become infeasible after a change and thus dominate other feasible solutions of greater cardinality without updating evaluations. This can be circumvented by assuming that the changes' directions are the same every time, i.e., $(d^*_i-d_i)(d^*_j-d_j)\geq0$ for every $(i,j)$ pair. Instead of imposing assumptions on the problems, we only assume scenarios where POMC successfully detects and responds to changes. This allows us to focus our analysis entirely on the algorithm's adaptation efficiency under arbitrary constraint threshold change scenarios.

\section{Runtime Analysis}

For the runtime analysis of POMC for static problems, we refer to the results by~\citet{Do2020}. In short, its worst-case approximation ratios on static problems are comparable to those of the classical GREEDY algorithm, assuming submodularity and weak monotonicity in the objective function \cite{GreedyMaxBoundCurvPartMatroid}. On the other hand, a direct comparison in other cases is not straightforward as the bounds involve different sets of parameters.

The strength of POMC in dynamic constraints handling lies in the fact that it stores a good solution for each cardinality level up to $\bar{d}$. In this way, when $\bar{d}$ changes, the population will contain good solutions for re-optimization. We use the concept of greedy addition $v^*_X$ to a solution $X$.
\[v^*_X=\argmax_{v\in V\setminus X}f_1(X\cup\{v\}).\]

It can be shown that for any $X$ where $|X|<\bar{d}$, the corresponding greedy addition $v^*_X$ w.r.t. $(f,B,D)$ is the same as the one w.r.t. $(f,B,D^*)$ if $\bar{d}^*\geq\bar{d}$, since $X\cup\{v\}$ is still feasible for all $v\in V\setminus X$. Thus, we can derive the following result from Lemma 2 in \cite{MaxMonoApproxSubmodMulti}, and Lemma 1 in \cite{Do2020}.

\begin{lemma}\label{greed_prog_submodular_dyn}
Let $f$ be a submodular function, $\bar{d}^*\geq\bar{d}$, and $\epsilon_d$ be defined in Definition \ref{mono_approx}, for all $X\subseteq V$ such that $|X|=j<\bar{d}$, we have both
\begin{align*}
&f(X\cup\{v^*_X\})-f(X)\geq\frac{1}{d}\left[f(OPT)-f(X)-j\epsilon_{d+j+1}\right],\\
&f(X\cup\{v^*_X\})-f(X)\geq\frac{1}{d^*}\left[f(OPT^*)-f(X)-j\epsilon_{d^*+j+1}\right].
\end{align*}
\end{lemma}
\begin{proof}
The first part is from Lemma 1 in \cite{Do2020}, while the second follows since if $|X|<\bar{d}\leq\bar{d}^*$, then the element contributing the greedy marginal gain to $f_1(X)$ is unchanged.
\end{proof}

The result carries over due to satisfied assumption $|X|<\bar{d}^*$. Using these inequalities, we can construct a proof, following a similar strategy as the one for Theorem 5 in \cite{ParetoSubsetDynCon} which leads to the following theorem.

\begin{theorem}\label{gsemo_submodular_dyn}
For the problem of maximizing a submodular function under partition matroid constraints, assuming $\bar{d}^*\geq\bar{d}$, POMC generates a population $P$ in expected run time $\mathcal{O}(d^2n/k)$ such that
\begin{align}\label{eq:gsemo_submodular_dyn}
\begin{split}
&\forall m\in[0,\bar{d}],\exists X\in P,|X|\leq m\\&\wedge f(X)\geq\left[1-\left(1-\frac{1}{d}\right)^{m}\right][f(OPT)-(m-1)\epsilon_{d+m}]\\&\wedge f(X)\geq\left[1-\left(1-\frac{1}{d^*}\right)^{m}\right][f(OPT^*)-(m-1)\epsilon_{d^*+m}].
\end{split}
\end{align}
\end{theorem}
\begin{proof}
Let $S(X,j)$ be the expression
\begin{align*}
&|X|\leq j\\&\wedge f(X)\geq\left[1-\left(1-\frac{1}{d}\right)^{j}\right][f(OPT)-(j-1)\epsilon_{d+j}]\\&\wedge f(X)\geq\left[1-\left(1-\frac{1}{d^*}\right)^{j}\right][f(OPT^*)-(j-1)\epsilon_{d^*+j}],
\end{align*}
and $Q(i)$ be the expression $\forall j\in[0,i],\exists X\in P_t,S(X,j)$. We have that $Q(0)$ holds. For each $h\in[0,\bar{d}-1]$, assume $Q(h)$ holds and $Q(h+1)$ does not hold at some iteration $t$. Let $X\in P_t$ be the solution such that $S(X,h)$ holds, $S(Y,h)$ and $Q(h)$ holds at iteration $t+1$ for any solution $Y\in P_{t+1}$ such that $Y\succeq X$. This means once $Q(h)$ holds, it must hold in all subsequent iterations of POMC. Let $X^*=X\cup\{v^*_X\}$, Lemma \ref{greed_prog_submodular_dyn} implies
\begin{align*}
f(X^*)&\geq\left[1-\left(1-\frac{1}{d}\right)^{h+1}\right][f(OPT)-h\epsilon_{d+h+1}],\\
\text{and}&\\
f(X^*)&\geq\frac{1}{d^*}f(OPT^*)-\frac{h}{d^*}\epsilon_{d^*+h+1}\\&+\left(1-\frac{1}{d^*}\right)\left[1-\left(1-\frac{1}{d^*}\right)^h\right]\\&[f(OPT^*)-(h-1)\epsilon_{d^*+h}]\\&
\geq\left[1-\left(1-\frac{1}{d^*}\right)^{h+1}\right][f(OPT^*)-h\epsilon_{d^*+h+1}].
\end{align*}
The second inequality uses $0\leq\epsilon_{d^*+h}\leq\epsilon_{d^*+h+1}$. This means that $S(X^*,h+1)$ holds. Therefore, if $X^*$ is generated, then $Q(h+1)$ holds, regardless of whether $X^*$ is dominated afterwards. According to the bit flip procedure, $X^*$ is generated with probability at least $1/(en(d+1))$ which implies
\[\Pr[Q(h+1)|Q(h)]\geq\frac{1}{en(d+1)},\text{ and }\Pr[\neg Q(h)|Q(h)]=0.\]

Using the additive drift theorem \cite{DriftAnalysis}, the expected number of iterations until $Q(\bar{d})$ holds is at most $e\bar{d}n(d+1)=\mathcal{O}(d^2n/k)$. This completes the proof.
\end{proof}

The statement \eqref{eq:gsemo_submodular_dyn} implies the following results.
\begin{align*}
&\forall m\in[0,\bar{d}],\exists X\in P,|X|\leq m\\&\wedge f(X)\geq\left(1-e^{m/d}\right)\left[f(OPT)-(m-1)\epsilon_{d+m}\right]\\&\wedge f(X)\geq\left(1-e^{m/d^*}\right)\left[f(OPT^*)-(m-1)\epsilon_{d^*+m}\right].
\end{align*}

We did not put this more elegant form directly in Theorem \ref{gsemo_submodular_dyn} since it cannot be used in the subsequent proof; only Expression \eqref{eq:gsemo_submodular_dyn} is applicable. Note that the result also holds for $\bar{d}^*\leq\bar{d}$ if we change the quantifier to $\forall m\in[0,\bar{d}^*]$. It implies that when a change such that $\bar{d}^*\leq\bar{d}$ occurs after cubic run time, POMC is likely to instantly satisfy the new approximation ratio bound, which would have taken it extra cubic run time to achieve if restarted. Therefore, it adapts well in such cases, assuming sufficient run time is allowed between changes. On the other hand, if $\bar{d}^*>\bar{d}$, the magnitude of the increase affects the difficulty with which the ratio can be maintained. The result also states a ratio bound w.r.t. the new optimum corresponding to the new constraint thresholds. As we will show using this statement, by keeping the current population (while discarding infeasible solutions), POMC can adapt to the new optimum quicker than it can with the restart strategy.

\begin{theorem}\label{gsemo_submodular_dyn2}
Assuming POMC achieves a population satisfying \eqref{eq:gsemo_submodular_dyn}, after the change where $\bar{d}^*>\bar{d}$, POMC generates in expected time $\mathcal{O}((\bar{d}^*-\bar{d})d^*n)$ a solution $X$ such that
\[f(X)\geq\left(1-e^{\bar{d}^*/d^*}\right)\left[f(OPT^*)-(\bar{d}^*-1)\epsilon_{d^*+\bar{d}^*}\right].\]
\end{theorem}
\begin{proof}
Let $S(X,i)$ be the expression
\begin{align*}
&|X|\leq j\\&\wedge f(X)\geq\left[1-\left(1-\frac{1}{d^*}\right)^{j}\right][f(OPT^*)-(j-1)\epsilon_{d^*+j}].
\end{align*}
Assuming \eqref{eq:gsemo_submodular_dyn} holds for some iteration $t$, let $\bar{X}\in P_t$ be a solution such that $S(\bar{X},i)$ holds for some $i\in[\bar{d},\bar{d}^*)$, and $v^*_{\bar{X}}=\argmax_{v\in V\setminus\bar{X}}\{f_1(\bar{X}\cup\{v\})-f_1(\bar{X})\}$ w.r.t. $(f,B,D^*)$ for any $\bar{X}\subset V$, and $\bar{X}'=\bar{X}\cup\{v^*_{\bar{X}}\}$, Lemma \ref{greed_prog_submodular_dyn} implies
\begin{align*}
f(\bar{X}')&\geq\left[1-\left(1-\frac{1}{d^*}\right)^{i+1}\right][f(OPT^*)-i\epsilon_{d^*+i+1}].
\end{align*}
Hence, $S(\bar{X}',i+1)$ holds. It is shown that such a solution is generated by POMC with probability at least $1/(en(d^*+1))$. Also, for any solution $X$ satisfying $S(X,i)$, another solution $Y\succeq X$ must also satisfy $S(Y,i)$. Therefore, the Additive Drift Theorem implies that given $S(X,\bar{d})$ holds for some $X$ in the population, POMC generates a solution $Y$ satisfying $S(Y,\bar{d}^*)$ in expected time at most $(\bar{d}^*-\bar{d})en(d^*+1)=\mathcal{O}((\bar{d}^*-\bar{d})d^*n)$. Such a solution satisfies the inequality in the theorem.
\end{proof}

The degree with which $\bar{d}$ increases only contributes linearly to the run time bound, which shows efficient adaptation to the new search space.

For the monotone objective functions cases, without loss of generality, we assume that $f$ is normalized ($f(\emptyset)=0$). We make use of the following inequalities, derived from Lemma 1 in \cite{PPOSS}, and Lemma 2 in \cite{Do2020}, using the same insight as before.

\begin{lemma}\label{greed_prog_monotone_dyn}
Let $f$ be a monotone function, $\bar{d}^*\geq\bar{d}$, and $\gamma_{i,j}$ be defined in Definition \ref{submodratio_def}, for all $X\subseteq V$ such that $|X|=j<\bar{d}$, we have both
\begin{align*}
&f(X\cup\{v^*_X\})-f(X)\geq\frac{\gamma_{j+1,d}}{d}[f(OPT)-f(X)],\\
&f(X\cup\{v^*_X\})-f(X)\geq\frac{\gamma_{j+1,d^*}}{d^*}[f(OPT^*)-f(X)].
\end{align*}
\end{lemma}
\begin{proof}
The first part is from Lemma 2 in \cite{Do2020}, while the second follows since if $|X|<\bar{d}\leq\bar{d}^*$, then the element contributing the greedy marginal gain to $f_1(X)$ is unchanged.
\end{proof}

This leads to the following result which show POMC's capability in adapting to changes where $\bar{d}^*<\bar{d}$.

\begin{theorem}\label{gsemo_monotone_dyn}
For the problem of maximizing a monotone function under partition matroid constraints, assuming $\bar{d}^*\geq\bar{d}$, POMC generates a population $P$ in expected run time $\mathcal{O}(d^2n/k)$ such that
\begin{align}\label{eq:gsemo_monotone_dyn}
\begin{split}
&\forall m\in[0,\bar{d}],\exists X\in P,|X|\leq m\\&\wedge f(X)\geq\left[1-\left(1-\frac{\gamma_{m,d}}{d}\right)^{m}\right]f(OPT)\\&\wedge f(X)\geq\left[1-\left(1-\frac{\gamma_{m,d^*}}{d^*}\right)^{m}\right]f(OPT^*).
\end{split}
\end{align}
\end{theorem}
\begin{proof}
Let $S(X,j)$ be the expression
\begin{align*}
|X|\leq j&\wedge f(X)\geq\left[1-\left(1-\frac{\gamma_{j,d}}{d}\right)^{j}\right]f(OPT)\\&\wedge f(X)\geq\left[1-\left(1-\frac{\gamma_{j,d^*}}{d^*}\right)^{j}\right]f(OPT^*),
\end{align*}
and $Q(i)$ be the expression $\forall j\in[0,i],\exists X\in P_t,S(X,j)$. We have that $Q(0)$ holds. Similar to the proof for Theorem \ref{gsemo_submodular_dyn}, we get, using Lemma \ref{greed_prog_monotone_dyn}, that the expected number of iterations of POMC until $Q(\bar{d})$ holds is at most $\mathcal{O}(d^2n/k)$.
\ignore{
For each $h\in[0,\bar{d}-1]$, assume $Q(h)$ holds and $Q(h+1)$ does not hold at some iteration $t$, let $X\in P_t$ be the solution such that $S(X,h)$ holds, $S(Y,h)$ and $Q(h)$ holds at iteration $t+1$ for any solution $Y\in P_{t+1}$ such that $Y\succeq X$. This means once $Q(h)$ holds, it must hold in all subsequent iterations of POMC. Let $X^*=X\cup\{v^*_X\}$, Lemma \ref{greed_prog_monotone_dyn} implies
\begin{align*}
f(X^*)&\geq\left[1-\left(1-\frac{\gamma_{h,d}}{d}\right)^{h+1}\right]f(OPT),\\
\text{and}&\\
f(X^*)&\geq\frac{\gamma_{h,d^*}}{d^*}f(OPT^*)\\&+\left(1-\frac{\gamma_{h+1,d^*}}{d^*}\right)\left[1-\left(1-\frac{\gamma_{h,d^*}}{d^*}\right)^h\right]f(OPT^*)\\&
\geq\left[1-\left(1-\frac{\gamma_{h+1,d^*}}{d^*}\right)^{h+1}\right]f(OPT^*).
\end{align*}
The second inequality uses $\gamma_{h,d^*}\geq\gamma_{h+1,d^*}$. This means $S(X^*,h+1)$ holds. Therefore, if $X^*$ is generated, then $Q(h+1)$ holds, regardless of whether $X^*$ is dominated afterwards. According to the bit flip procedure, $X^*$ is generated with probability at least $1/en(d+1)$, so
\[\Pr[Q(h+1)|Q(h)]\geq\frac{1}{en(d+1)},\quad\text{and}\quad\Pr[\neg Q(h)|Q(h)]=0.\]
The Additive Drift Theorem implies that assuming $Q(0)$, the expected number of iterations before $Q(\bar{d})$ holds is at most $e\bar{d}n(d+1)=\mathcal{O}(d^2n/k)$. This proves the theorem.}
\end{proof}

Similar to Expression~\eqref{eq:gsemo_submodular_dyn}, Expression~\eqref{eq:gsemo_monotone_dyn} gives a more elegant form.
\begin{align*}
&\forall m\in[0,\bar{d}],\exists X\in P,|X|\leq m\\&\wedge f(X)\geq\left(1-e^{-\gamma_{m,d}m/d}\right)f(OPT)\\&\wedge f(X)\geq\left(1-e^{-\gamma_{m,d^*}m/d^*}\right)f(OPT^*).
\end{align*}

Just like Theorem \ref{gsemo_submodular_dyn}, this result holds for $\bar{d}^*\leq\bar{d}$ when the quantifier is $\forall m\in[0,\bar{d}^*]$. It is implied that if such a population is made, the new approximation ratio bound is immediately satisfied when the thresholds decrease. Once again, we can derive the result on POMC's adapting performance to $\bar{d}^*\geq\bar{d}$ by using Theorem \ref{gsemo_monotone_dyn}.

\begin{theorem}\label{gsemo_monotone_dyn2}
Assuming POMC achieves a population satisfying \eqref{eq:gsemo_monotone_dyn}, after the change where $\bar{d}^*>\bar{d}$, POMC generates in expected time $\mathcal{O}((\bar{d}^*-\bar{d})d^*n)$ a solution $X$ such that
\[f(X)\geq\left(1-e^{-\gamma_{\bar{d}^*,d^*}\bar{d}^*/d^*}\right)f(OPT^*).\]
\end{theorem}
\begin{proof}
Let $S(X,i)$ be a expression
\[|X|\leq j\wedge f(X)\geq\left[1-\left(1-\frac{\gamma_{\bar{d}^*,d^*}}{d^*}\right)^{j}\right]f(OPT^*).\]
Assuming \eqref{eq:gsemo_monotone_dyn} holds for some iteration $t$, let $\bar{X}\in P_t$ be a solution such that $S(\bar{X},i)$ holds for some $i\in[\bar{d},\bar{d}^*)$, and $v^*_{\bar{X}}=\argmax_{v\in V\setminus\bar{X}}\{f_1(\bar{X}\cup\{v\})-f_1(\bar{X})\}$ w.r.t. $(f,B,D^*)$ for any $\bar{X}\subset V$, and $\bar{X}'=\bar{X}\cup\{v^*_{\bar{X}}\}$, Lemma \ref{greed_prog_monotone_dyn} implies
\begin{align*}
f(\bar{X}')&\geq\left[1-\left(1-\frac{\gamma_{i+1,d^*}}{d^*}\right)^{i+1}\right]f(OPT^*).
\end{align*}
So $S(\bar{X}',i+1)$ holds. Following the same reasoning in the proof for Theorem \ref{gsemo_submodular_dyn2}, we get that POMC, given a solution $X$ in the population satisfying $S(X,\bar{d})$, generates another solution $Y$ satisfying $S(Y,\bar{d}^*)$ in expected run time at most $\mathcal{O}((\bar{d}^*-\bar{d})d^*n)$.
\end{proof}

Theorem \ref{gsemo_submodular_dyn2} and Theorem \ref{gsemo_monotone_dyn2} state that there is at least one solution in the population that satisfies the respective approximation ratio bounds in each case, after POMC is run for at least some expected number of iterations. However, the proofs for Theorem \ref{gsemo_submodular_dyn} and Theorem \ref{gsemo_monotone_dyn} also imply that the same process generates, for each cardinality level from $0$ to $\bar{d}^*$, a solution that also satisfies the respective approximation bound ratio, adjusted to the appropriate cardinality term. These results imply that instead of restarting, maintaining the non-dominated population provides a better starting point to recover relative approximation quality in any case. This suggests that the inherent diversity in Pareto-optimal fronts is suitable as preparation for changes in constraint thresholds. 

\section{Experimental Investigations}

We compare the POMC algorithm against GREEDY \cite{GreedyMaxBoundCurvPartMatroid} with the restart approach, on undirected max cut problems with random graphs. Given a graph $G=(V,E,c)$ where $c$ is a non-negative edge weight function, the goal is to find, while subjected to changing partition matroid constraints,
\[X^*=\argmax_X\left\lbrace\sum_{a\in X,b\in V\setminus X}c(a,b)\right\rbrace.\]

Weighted graphs are generated for the experiments based on two parameters: number of vertices ($n$) and edge density. Here, we consider $n=200$, and 5 density values: 0.01, 0.02, 0.05, 0.1, 0.2. For each $n$-$density$ pair, a different weighted graph is randomly generated with the following procedure:
\begin{enumerate}
    \item Randomly sample $E$ from $V\times V$ without replacement, until $|E|=\lfloor density\times n^2\rfloor$.
    \item Assign to $c(a,b)$ a uniformly random value in $[0,1]$ for each $(a,b)\in E$.
    \item Assign $c(a,b)=0$ for all $(a,b)\notin E$.
\end{enumerate}

To limit the variable dimensions so that the results are easier to interpret, we only use one randomly chosen graph for each $n$-density pair. We also consider different numbers of partitions $k$: 1, 2, 5, 10. For $k>1$, each element is assigned to a partition randomly. Also, the sizes of the partitions are equal, as with the corresponding constraint thresholds.
\begin{figure}[t!]
\centering
\includegraphics[width=.8\linewidth]{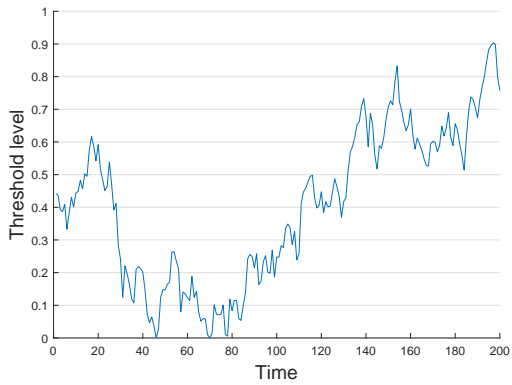}
\caption{Threshold level over time for dynamic problems}
\label{fig:thres_changes}
\end{figure}

The dynamic component is the thresholds $d_i$. We use the approach outlined in \cite{DynKnapsack} to generate threshold changes in the form of a sequence of $m$ values $\{b_i\}_{i=1}^m$. In particular, these values range in $[0,1]$ and are applied to each instance: $d_i=\max\{b_j|B_i|,1\}$ at the $j^{th}$ change, rounded to the nearest integer. The generating formulas are as follow:
\[b_1=\mathcal{U}(0,1),\quad b_{i+1}=\max\{\min\{b_i+\mathcal{N}(0,0.05^2),1\},0\}.\]
The values $b_i$ used in the experiments are displayed in Figure \ref{fig:thres_changes} where the number of changes is $m=200$.

For each change, the output of GREEDY is obtained by running without restriction on evaluations. These values are used as baselines to compare against POMC's outputs . POMC is run with three different settings for the number of evaluations between changes: $5000$, $10000$, $20000$. Smaller numbers imply a higher change frequency leading to a higher difficulty in adapting to changes. Furthermore, POMC is run 30 times for each setting, and means and standard deviations of results are shown in POMC$_{5000}$, POMC$_{10000}$, POMC$_{20000}$, corresponding to different change frequency settings.

\begin{table*}[t!]
\begin{small}
\renewcommand{\arraystretch}{.76}
\caption{Experimental results for dynamic max cut with $n=200$. Outputs are shown in batches of 50 threshold changes. L--W--T are numbers of losses, wins, ties POMC has over GREEDY, determined by U-tests on data in each change.}
\label{table:experimentsd}
\begin{center}
\begin{tabular}{lll@{\hspace{6pt}}c@{\hspace{6pt}}c@{\hspace{6pt}}c@{\hspace{6pt}}c@{\hspace{6pt}}c@{\hspace{6pt}}c@{\hspace{6pt}}c@{\hspace{6pt}}c@{\hspace{6pt}}c@{\hspace{6pt}}c@{\hspace{6pt}}c}
\toprule
\multirow{2}{*}{$k$}&\multirow{2}{*}{density}&\multirow{2}{*}{Changes}&\multicolumn{2}{c}{GREEDY}&\multicolumn{3}{c}{POMC$_{5000}$}&\multicolumn{3}{c}{POMC$_{10000}$}&\multicolumn{3}{c}{POMC$_{20000}$}\\\cmidrule(l{2pt}r{2pt}){4-5}\cmidrule(l{2pt}r{2pt}){6-8}\cmidrule(l{2pt}r{2pt}){9-11}\cmidrule(l{2pt}r{2pt}){12-14}&&&mean&std&mean&std&L--W--T&mean&std&L--W--T&mean&std&L--W--T\\\cmidrule(l{2pt}r{2pt}){1-14}
\multirow{12}{*}{1}&\multirow{4}{*}{0.01}&1--50&82.93&27.5&81.26&27.01&26--17--7&82.85&27.49&17--25--8&83.68&27.81&7--31--12\\&
&51--100&57.17&24.63&56.11&23.79&39--2--9&56.88&24.26&26--7--17&57.29&24.57&13--16--21\\&
&101--150&100.7&5.001&100.9&5.535&12--32--6&101.7&5.255&5--43--2&102.3&5.11&1--48--1\\&
&151--200&103.1&1.436e-14&105.1&0.2395&0--50--0&105.1&0.2814&0--50--0&105.2&0.2764&0--50--0\\\cmidrule(l{2pt}r{2pt}){2-14}
&\multirow{4}{*}{0.05}&1--50&289.3&106.2&283.7&103.4&30--12--8&288.3&105.3&20--20--10&290.5&106.3&7--29--14\\&
&51--100&186.4&87.89&183.5&85.18&43--0--7&185.3&86.25&33--1--16&186.2&86.92&19--10--21\\&
&101--150&359.4&23.64&358.2&24.7&23--23--4&361.9&25.16&9--39--2&363.5&24.95&2--47--1\\&
&151--200&373.2&2.297e-13&379.5&1.618&0--50--0&380.6&1.817&0--50--0&379.5&1.344&0--50--0\\\cmidrule(l{2pt}r{2pt}){2-14}
&\multirow{4}{*}{0.2}&1--50&890.3&345.2&877.6&336.8&31--13--6&886.8&341&23--21--6&892&343.8&12--28--10\\&
&51--100&551.7&274.3&545.1&266.1&41--2--7&548.7&268.9&34--5--11&550.6&270.5&28--9--13\\&
&101--150&1121&79.61&1116&81.18&33--16--1&1123&81.52&18--25--7&1126&81.89&8--41--1\\&
&151--200&1167&4.594e-13&1179&1.773&0--50--0&1180&1.592&0--50--0&1181&1.734&0--50--0\\\cmidrule(l{2pt}r{2pt}){1-14}
\multirow{12}{*}{2}&\multirow{4}{*}{0.01}&1--50&83&26.99&81.06&26.38&29--14--7&82.77&26.93&20--20--10&83.65&27.27&7--30--13\\&
&51--100&56.89&24.54&55.62&23.61&42--0--8&56.45&24.08&32--3--15&56.91&24.39&17--13--20\\&
&101--150&100.7&4.921&100.5&5.562&18--27--5&101.5&5.351&5--39--6&102.2&5.136&1--48--1\\&
&151--200&103.1&1.436e-14&105.2&0.2744&0--50--0&105.1&0.3222&0--50--0&105.2&0.3047&0--50--0\\\cmidrule(l{2pt}r{2pt}){2-14}
&\multirow{4}{*}{0.05}&1--50&289.4&104.8&283.2&101.9&37--9--4&288&103.9&26--18--6&290.7&105.1&13--30--7\\&
&51--100&185.5&87.7&181.9&84.58&44--0--6&184&85.82&36--1--13&185.1&86.55&25--9--16\\&
&101--150&359.6&23.24&357.8&24.65&27--21--2&361.4&24.79&12--35--3&363.3&24.71&2--44--4\\&
&151--200&373.2&2.297e-13&379.5&1.547&0--50--0&380.2&1.603&0--50--0&379.6&1.72&0--50--0\\\cmidrule(l{2pt}r{2pt}){2-14}
&\multirow{4}{*}{0.2}&1--50&890.5&341.2&876.4&332.2&33--8--9&886.5&337&24--19--7&891.8&339.6&15--26--9\\&
&51--100&548.6&274&540.8&265&42--1--7&544.9&268.2&40--5--5&547.2&270.1&28--7--15\\&
&101--150&1121&78.79&1115&80.14&32--15--3&1122&80.65&19--25--6&1126&80.84&6--41--3\\&
&151--200&1167&4.594e-13&1178&2.186&0--50--0&1180&1.777&0--50--0&1181&1.365&0--50--0\\\cmidrule(l{2pt}r{2pt}){1-14}
\multirow{12}{*}{5}&\multirow{4}{*}{0.01}&1--50&82.89&26.89&80.18&26.47&35--10--5&82.03&26.82&28--17--5&83.19&27.09&21--22--7\\&
&51--100&57.45&23.4&55.25&22.07&46--0--4&56.51&22.63&44--0--6&57.13&22.94&35--2--13\\&
&101--150&100.3&5.205&98.99&6.439&31--15--4&100.4&6.181&17--26--7&101.3&5.851&7--38--5\\&
&151--200&103.1&1.436e-14&105.1&0.299&0--50--0&105.2&0.3053&0--50--0&105.3&0.2958&0--50--0\\\cmidrule(l{2pt}r{2pt}){2-14}
&\multirow{4}{*}{0.05}&1--50&288.1&104.2&276&99.2&45--2--3&282&101&40--7--3&286.2&102.6&35--10--5\\&
&51--100&185.9&83.57&179.9&79.17&44--3--3&183&80.87&41--5--4&184.9&81.91&31--13--6\\&
&101--150&357.1&25.23&347.7&27.29&41--2--7&353.3&26.81&35--13--2&357.4&26.77&23--20--7\\&
&151--200&373.1&0.1956&376&4.098&4--38--8&377.8&3.594&2--44--4&379.5&2.772&0--48--2\\\cmidrule(l{2pt}r{2pt}){2-14}
&\multirow{4}{*}{0.2}&1--50&891.6&337.8&868.9&324.9&42--2--6&881.1&329.8&40--7--3&888.7&333.4&27--10--13\\&
&51--100&556.6&262.2&546.6&251.8&45--0--5&552.3&256&38--1--11&555.2&258.4&26--5--19\\&
&101--150&1117&82.61&1101&84.57&37--10--3&1110&83.79&33--14--3&1117&83.62&25--20--5\\&
&151--200&1167&4.594e-13&1175&4.725&0--48--2&1178&4.139&0--50--0&1180&2.777&0--50--0\\\cmidrule(l{2pt}r{2pt}){1-14}
\multirow{12}{*}{10}&\multirow{4}{*}{0.01}&1--50&80.91&27.29&77.03&26.71&41--3--6&79.3&27.08&32--10--8&80.69&27.32&20--17--13\\&
&51--100&54.9&21.54&51.59&19.49&50--0--0&53.23&20.3&47--0--3&54.17&20.78&38--1--11\\&
&101--150&99.72&5.231&97.23&7.459&33--11--6&99.15&6.985&20--24--6&100.3&6.261&12--34--4\\&
&151--200&103.1&0.09412&105&0.3504&0--50--0&105.1&0.3296&0--50--0&105.2&0.2794&0--50--0\\\cmidrule(l{2pt}r{2pt}){2-14}
&\multirow{4}{*}{0.05}&1--50&282.7&105.5&268.5&99.61&50--0--0&275.1&101.5&43--2--5&279.9&103.2&32--9--9\\&
&51--100&179.9&77.93&172.3&72.08&49--0--1&176.4&74.58&43--2--5&178.9&76.24&27--10--13\\&
&101--150&356.2&23.7&344.8&28.83&44--1--5&351.2&27.92&37--11--2&355.8&26.49&21--21--8\\&
&151--200&372.6&0.7456&374.1&4.464&7--36--7&376.3&3.334&2--45--3&377.5&2.62&1--49--0\\\cmidrule(l{2pt}r{2pt}){2-14}
&\multirow{4}{*}{0.2}&1--50&877.9&340.8&850.1&324.3&45--4--1&865.3&330.6&36--8--6&873.5&334.4&24--17--9\\&
&51--100&544.6&247&532&233.3&41--4--5&539.2&238.8&30--12--8&543.5&242.3&17--21--12\\&
&101--150&1116&77.15&1093&82.69&45--2--3&1106&81.16&36--10--4&1115&79.89&22--21--7\\&
&151--200&1166&1.281&1167&9.098&8--32--10&1173&6.941&2--44--4&1176&5.929&1--48--1\\\cmidrule(l{2pt}r{2pt}){1-14}

\end{tabular}
\end{center}
\end{small}
\end{table*}

We show the aggregated results in Table \ref{table:experimentsd} where a U-test \cite{Corder09} with 95\% confidence interval is used to determine statistical significance in each change. 
%
The results show that outputs from both algorithms are very closely matched most of the time, with the greatest differences observed in low graph densities. Furthermore, we expect that GREEDY fairs better against POMC when the search space is small (low threshold levels). While this is observable, the opposite phenomenon when the threshold level is high can be seen more easily.

We see that during consecutive periods of high constraint thresholds, POMC's outputs initially fall behind GREEDY's, only to overtake them at later changes. This suggests that POMC rarely compromises its best solutions during those periods as the consequence of the symmetric submodular objective function. It also implies that restarting POMC from scratch upon a change would have resulted in significantly poorer results. On the other hand, POMC's best solutions follow GREEDY's closely during low constraint thresholds periods. This indicates that by maintaining feasible solutions upon changes, POMC keeps up with GREEDY in best objectives well within quadratic run time.

Comparing outputs from POMC with different interval settings, we see that those from runs with higher number of evaluations between changes are always better. However, the differences are minimal during the low constraint thresholds periods. This aligns with our theoretical results in the sense that the expected number of evaluations needed to guarantee good approximations depends on the constraint thresholds. As such, additional evaluations won't yield significant improvements within such small feasible spaces.

Comparing between different $k$ values, POMC seems to be at a disadvantage against GREEDY's best at increased $k$. This is expected since more partitions leads to more restrictive feasible search spaces, given everything else is unchanged, and small feasible spaces amplify the benefit of each greedy step. Nevertheless, POMC does not seem to fall behind GREEDY significantly for any long period, even when given few resources.

\section{Conclusions}

In this study, we have considered combinatorial problems with dynamic constraint thresholds, particularly the important classes of problems where the objective functions are submodular or monotone. We have contributed to the theoretical run time analysis of a Pareto optimization approach on such problems. Our results indicate POMC's capability of maintaining populations to efficiently adapt to changes and preserve good approximations. In our experiments, we have shown that POMC is able to maintain at least the greedy level of quality and often even obtains better solutions.

\section{Acknowledgements}
This work has been supported by the Australian Research Council through grants DP160102401 and DP190103894.

\bibliography{ref}
\end{document}